\documentclass[letterpaper, 10 pt, conference]{ieeeconf}  

\IEEEoverridecommandlockouts                              

\usepackage{times}
\usepackage{graphicx}
\usepackage{adjustbox}
\usepackage[table]{xcolor}
\usepackage{amsmath} 

\usepackage{amssymb}  
\usepackage{amsthm}
\usepackage{todonotes}
\usepackage[linesnumbered,ruled,noend]{algorithm2e}
\usepackage{etoolbox}
\makeatletter
\patchcmd{\@algocf@start}
  {-1.5em}
  {0pt}
  {}{}
\makeatother
\usepackage{scalerel,stackengine}
\stackMath
\newcommand\reallywidehat[1]{%
\savestack{\tmpbox}{\stretchto{%
  \scaleto{%
    \scalerel*[\widthof{\ensuremath{#1}}]{\kern-.6pt\bigwedge\kern-.6pt}%
    {\rule[-\textheight/2]{1ex}{\textheight}}
  }{\textheight}%
}{0.5ex}}%
\stackon[1pt]{#1}{\tmpbox}%
}

\SetCommentSty{mycommfont}

\newtheorem{lem}{Lemma}
\newtheorem{theorem}{Theorem}
\newtheorem{problem}{Problem}

\usepackage{multicol}
\usepackage[bookmarks=true]{hyperref}

\newcommand{\X}{\mathcal{X}}
\newcommand{\U}{\mathcal{U}}
\renewcommand{\S}{\mathcal{S}}
\newcommand{\B}{\mathcal{B}}
\newcommand{\Xsafe}{\X_{\textrm{safe}}}
\newcommand{\Xunsafe}{\X_{\textrm{unsafe}}}
\newcommand{\T}{\mathcal{T}}
\newcommand{\tx}{\bar{x}}
\newcommand{\tu}{\bar{u}}

\newcommand{\energy}{\mathcal{E}}

\newcommand{\fb}{\textrm{fb}}
\newcommand{\I}{\mathbf{I}}
\newcommand{\f}{h}
\newcommand{\g}{g}
\newcommand{\gzero}{f}
\newcommand{\gone}{B}
\newcommand{\logb}{\texttt{logb}}
\newcommand{\lmtd}{\textrm{LMTCD-RRT }}
\newcommand{\proj}{\textrm{proj}}

\title{\LARGE \bf
Model Error Propagation via Learned Contraction Metrics for \\ Safe Feedback Motion Planning of Unknown Systems 
}

\author{Glen Chou, Necmiye Ozay, and Dmitry Berenson$^{1}$
\thanks{$^{1}$Electrical Engineering and Computer Science, University of Michigan, Ann Arbor, MI,
        {\tt\small \{gchou, necmiye, dmitryb\}@umich.edu}}%
}

\begin{document}

\maketitle
\thispagestyle{empty}
\pagestyle{empty}

\begin{abstract}

We present a method for contraction-based feedback motion planning of locally incrementally exponentially stabilizable systems with unknown dynamics that provides probabilistic safety and reachability guarantees. Given a dynamics dataset, our method learns a deep control-affine approximation of the dynamics. To find a trusted domain where this model can be used for planning, we obtain an estimate of the Lipschitz constant of the model error, which is valid with a given probability, in a region around the training data, providing a local, spatially-varying model error bound. We derive a trajectory tracking error bound for a contraction-based controller that is subjected to this model error, and then learn a controller that optimizes this tracking bound. With a given probability, we verify the correctness of the controller and tracking error bound in the trusted domain. We then use the trajectory error bound together with the trusted domain to guide a sampling-based planner to return trajectories that can be robustly tracked in execution. We show results on a 4D car, a 6D quadrotor, and a 22D deformable object manipulation task, showing our method plans safely with learned models of high-dimensional underactuated systems, while baselines that plan without considering the tracking error bound or the trusted domain can fail to stabilize the system and become unsafe.

\end{abstract}

\section{Introduction}

Provably safe motion planning algorithms for unknown systems are critical for deploying robots in the real world. While planners are reliable when the system dynamics are known exactly, the dynamics often may be poorly modeled or unknown. To address this, data-driven methods (i.e. model-based reinforcement learning) learn the dynamics from data and plan with the learned model. However, such methods can be unsafe, in part because the planner can and will exploit errors in the learned dynamics to return trajectories which cannot actually be tracked on the real system, leading to unpredictable, unsafe behavior when executed. Thus, to guarantee safety, it is of major interest to establish a bound on the error that the true system may see when attempting to track a trajectory planned with the learned dynamics, and to use it to guide the planning of robustly-trackable trajectories.

One key property of learned dynamics models is that they have varying error across the state space: they should be more accurate on the training data, and that accuracy should degrade when moving away from it. Thus, the model error that the system will see in execution will depend on the domain that it visits. This reachable domain also depends on the tracking controller; for instance, a poor controller will lead to the system visiting a larger set of possible states, and thus experiencing a larger possible model error. To analyze this, we need a bound on the trajectory tracking error for a given disturbance description (a tracking tube). In this paper, we consider tracking controllers based on contraction theory. Introduced in \cite{lohmiller} for autonomous systems and extended to the control-affine case in \cite{manchester}, control contraction theory studies the incremental stabilizability of a system, making it uniquely suited for obtaining trajectory tracking tubes under disturbance. In the past, tracking tubes have been derived for contraction-based controllers under simple uniform disturbance bounds \cite{sumeet_icra} (i.e. a UAV subject to wind with a known uniform upper bound). However, these assumptions are ill-suited for handling learned model error. Assuming a uniform disturbance bound over the space can be highly conservative, since the large model error far from the training data would yield enormous tracking tubes, rendering planning entirely infeasible. To complicate things further, obtaining an upper bound on the model error can be challenging, as we only know the value of the error on the training data.

To address this gap, we develop a method for safe contraction-based motion planning with learned dynamics models. In particular, our method is designed for high-dimensional neural network (NN) learned dynamics models, and provides probabilistic guarantees on safety and goal reachability for the true system. Our core insight is that we can derive a tracking error bound for a contraction-based controller under a spatially-varying model error description, and that we can use this error bound to bias planning towards regions in the state/control space where trajectories can be more robustly tracked. We summarize our contributions as:
\begin{itemize}
	\item A trajectory tracking error bound for contraction-based controllers subjected to a spatially-varying, Lipschitz constant-based model error bound that accurately reflects the learned model error.
	\item A deep learning framework for joint learning of dynamics, control contraction metrics (CCMs), and contracting controllers that are approximately optimized for planning performance under this model error description.
	\item A sampling-based planner that returns plans which can be safely tracked under the learned dynamics/controller.
	\item Evaluation of our method on learned dynamics up to 22D, and demonstrating that it outperforms baselines.
\end{itemize}

\section{Related Work}

Our work is related to contraction-based control of uncertain systems: \cite{sumeet_icra} applies contraction to feedback motion planning for systems with a known disturbance bound, while \cite{DBLP:journals/corr/abs-2004-01142, DBLP:journals/corr/abs-2003-10028} apply contraction to adaptive control under known model uncertainty structure, i.e. the uncertainty lies in the range of known basis functions. In this paper, the uncertainty arises from the error between the true dynamics and a learned NN approximation, which lacks such structure. It is also only known at certain states (the training data), making the disturbance bound \textit{a priori} unknown and nontrivial to obtain. These methods use sum-of-squares (SoS) optimization to find CCMs, which apply to moderate-dimensional polynomial systems \cite{sumeet_icra}, and cannot be used for NN models. Thus, \cite{dawei}, \cite{DBLP:journals/csysl/TsukamotoC21} model the CCM as an NN and learn it from data, assuming known dynamics subjected to disturbances with known uniform upper bound. Our method differs by learning the dynamics and CCM together to optimize planning performance under model error. Also related is \cite{sumeet_wafr}, which learns a CCM with the dynamics, but does not consider how the model error affects tracking.

More broadly, our work is related to safe learning-based control. Many methods learn stability certificates for a single equilbirium point \cite{DBLP:conf/nips/KolterM19,DBLP:journals/corr/abs-2008-05952}, but this is not especially useful for point-to-point motion planning. Other methods use Gaussian processes to bound the reachable tube of a trajectory \cite{koller2018learning} or safely explore a set \cite{akametalu2014reachability, berkenkamp2016safe}, but these methods assume a feedback controller is provided; we do not, as we learn a CCM-based controller. \cite{DBLP:journals/corr/abs-2002-01587} directly learns tracking tubes around trajectories, so plans must remain near \textit{trajectories} seen in training to be accurate; our method only requires plans to be near \textit{state/controls} seen in training. Finally, perhaps most relevant is \cite{lipschitz_ral}, which plans safely with learned dynamics by obtaining a tube around a plan inside a ``trusted domain". A restrictive key assumption of \cite{lipschitz_ral} is that the unknown system has as many control inputs as states. We remove this assumption in this work, requiring fundamental advancements in the method of \cite{lipschitz_ral}, i.e. in deriving a new tracking bound, controller, trusted domain, and planner.

\section{Preliminaries and Problem Statement}

\subsection{System models, notation, and differential geometry}

We consider deterministic unknown continuous-time nonlinear systems $\dot x = \f(x, u)$, where $\f: \X \times \U \rightarrow \X$, $\X \subseteq \mathbb{R}^{n_x}$, and $\U \subseteq \mathbb{R}^{n_u}$. We define $\g: \X \times \U \rightarrow \X$ to be a control-affine approximation of the true dynamics:

\begin{equation}\label{eq:g}
    \g(x,u) = \gzero(x) + \gone(x) u. 
\end{equation}

While we do not assume that the true dynamics are control-affine, we do assume that they are locally incrementally exponentially stabilizable, that is, there exists a $\beta$, $\lambda > 0$, and feedback controller such that $\Vert x^*(t) - x(t)\Vert \le \beta e^{-\lambda t} \Vert x^*(0) - x(0)\Vert$ for all solutions $x(t)$ in a domain. Many underactuated systems satisfy this, and it is much weaker than requiring $n_x = n_u$, as in \cite{lipschitz_ral}. Also, this only needs to hold in a task-relevant domain $D$, defined later. 

For a function $\eta$, a Lipschitz constant over a domain $\mathcal{Z}$ is any $L$ such that for all $z_1, z_2 \in \mathcal{Z}$, $\| \eta(z_1) - \eta(z_2) \| \leq L \| z_1 - z_2 \|$. Norms $\|\cdot\|$ are always the 2-norm. We define $L_{\f-\g}$ as the smallest Lipschitz constant of the error $\f - \g$. The argument of $\f-\g$ is a state-control pair $(x,u)$ and its value is a state. We define a ball $\B_r(x)$ as $\{y \enspace | \enspace \|y - x\| < r\}$, also referred to as a $r$-ball about $x$. We suppose the state space $\X$ is partitioned into safe $\Xsafe$ and unsafe $\Xunsafe$ sets (e.g., collision states). We denote $\widehat{Q} \doteq Q + Q^\top $ as a symmetrization operation on matrix $Q$, and $\bar\lambda(\hat Q)$ and $\underline\lambda(\hat Q)$ as its maximum and minimum eigenvalues, respectively. We overload notation when $Q(x)$ is a matrix-valued function, denoting $\bar\lambda_\mathcal{Q}(Q) \doteq \sup_{x\in \mathcal{Q}} \bar\lambda(Q(x))$ and $\underline\lambda_\mathcal{Q}(Q) \doteq \inf_{x\in \mathcal{Q}} \underline\lambda(Q(x))$. Let $\I_n$ be the identity matrix of size $n \times n$. Let $\mathbb{S}_n^{>0}$ denote the set of symmetric, positive definite $n \times n$ matrices. Let the Lie derivative of a matrix-valued function $Q(x) \in \mathbb{R}^{n\times n}$ along a vector $y \in \mathbb{R}^n$ be denoted as $\partial_y Q(x) \doteq \sum_{i=1}^n y^i \frac{\partial Q}{\partial x^i}$. Let $x^i$ denote the $i$th element of vector $x$. Let the notation $Q_{\perp}(x)$ refer to a basis for the null-space of matrix $Q(x)$.

Finally, we introduce the needed terminology from differential geometry. For a smooth manifold $\X$, a Riemannian metric tensor $M: \X \rightarrow \mathbb{S}_{n_x}^{>0}$ equips the tangent space $T_x \X$ at each element $x$ with an inner product $\delta_x^\top M(x) \delta_x$, providing a local length measure. Then, the length $l(c)$ of a curve $c: [0,1]\rightarrow \X$ between points $c(0)$, $c(1)$ can be computed by integrating the local lengths along the curve: $l(c) \doteq \int_0^1 \sqrt{V(c(s), c_s(s))} ds$, where for brevity $V(c(s), c_s(s)) \doteq c_s(s)^\top M(c(s))c_s(s)$, and $c_s(s) \doteq \partial c(s)/\partial s$. Then, the Riemann distance between two points $p,q\in\X$ can be defined as $\textrm{dist}(p,q) \doteq \inf_{c\in \mathcal{C}(p,q)} l(c)$, where $\mathcal{C}(p,q)$ is the set of all smooth curves connecting $p$ and $q$. Finally, we define the Riemann energy between $p$ and $q$ as $\energy(p,q) \doteq \textrm{dist}^2(p,q)$. 
\begin{figure*}
    \centering
    \includegraphics[width=\linewidth]{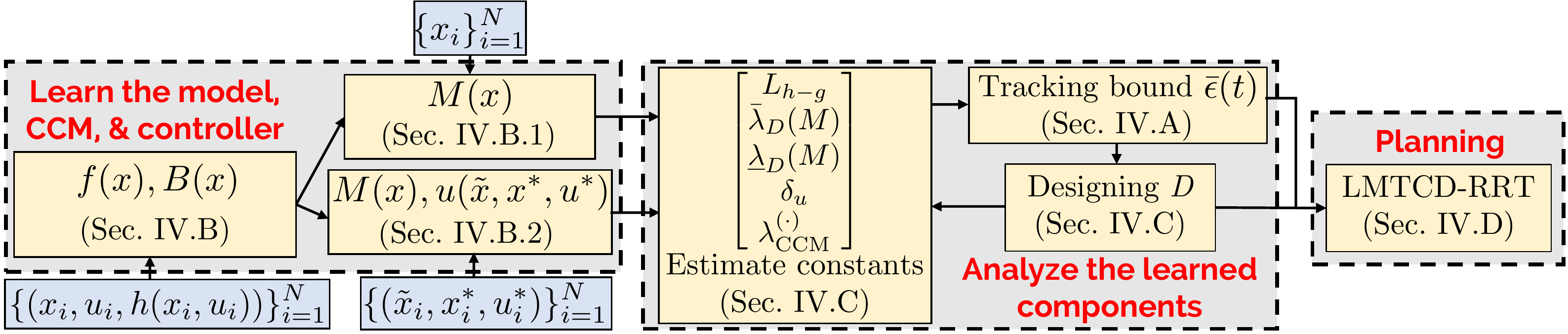}\vspace{-8pt}
    \caption{Method flowchart. \textbf{Left}: First, we learn a model of the dynamics using dataset $\S$ and obtain a contracting controller for this learned model (Prob. \ref{prob:learn}). \textbf{Center}: Next, within a trusted domain $D$, we verify the correctness of the controller, bound the model error, and bound the trajectory tracking error under this model error (Prob. \ref{prob:D}). \textbf{Right}: Finally, we use the error bounds to plan trajectories within $D$ that can be safely tracked in execution (Prob. \ref{prob:path}). }\vspace{-11pt}
    \label{fig:flowchart}
\end{figure*}
\subsection{Control contraction metrics (CCMs)}\label{sec:prelim_ccm}

Contraction theory studies how the distance between trajectories of a system changes with time to infer properties on incremental stability. This can be formalized with a contraction metric $M(x): \X \rightarrow \mathbb{S}_{n_x}^{>0}$ to measure if the differential distances between trajectories $V(x,\delta_x) = \delta_x^\top M(x) \delta_x$ shrink with time.
Control contraction metrics (CCMs) adapt this analysis to control-affine systems \eqref{eq:g}. For dynamics of the form \eqref{eq:g}, the differential dynamics can be written as $\dot{\delta}_x = (\frac{\partial \gzero}{\partial x} + \sum_{i=1}^{n_u} u^i \frac{\partial \gone^i}{\partial x})\delta_x + \gone(x) \delta_u$ \cite{sumeet_icra}, where $\gone^i(x)$ is the $i$th column of $\gone(x)$. Then, we call $M(x): \X \rightarrow \mathbb{S}_{n_x}^{>0}$ a CCM if there exists a differential controller $\delta_u$ such that the closed-loop system satisfies $\dot V(x, \delta_x) < 0$, for all $x$, $\delta_x$.

How do we find a CCM $M(x)$ ensuring the existence of $\delta_u$? First, define the dual metric $W(x) \doteq M^{-1}(x)$. Then, two sufficient conditions for contraction are \eqref{eq:strong}-\eqref{eq:weak} \cite{sumeet_icra, dawei}:

\begin{subequations}\label{eq:strong}\small
\begin{eqnarray}
	\gone_{\perp}(x)^\top \Big(-\partial_{\gzero} W(x) + \reallywidehat{\frac{\partial \gzero(x)}{\partial x} W(x)} + 2 \lambda W(x)\Big) \gone_{\perp}(x) \preceq 0 \label{eq:strong_con}\\
	\gone_{\perp}(x)^\top \Big( \partial_{\gone^j} W(x) - \reallywidehat{\frac{\partial \gone^j(x)}{\partial x} W(x)} \Big) \gone_{\perp}(x) = 0, j = 1...n_u\label{eq:strong_orth}
\end{eqnarray}
\end{subequations}
\begin{equation}\label{eq:weak}\small
	\dot M(x) + \reallywidehat{M(x)(A(x) + B(x)K(\tilde x,x^*,u^*))} + 2\lambda M(x) \prec 0
\end{equation}
where $A \doteq \frac{\partial f}{\partial x} + \sum_{i=1}^{n_u} u^i \frac{\partial \gone^i}{\partial x}$ and $K = \frac{\partial u(\tilde x, x^*,u^*)}{\partial x}$, where $u: \X\times\X\times\U \rightarrow \U$ is a feedback controller which takes as input the tracking deviation $\tilde x(t) \doteq x(t) - x^*(t)$ from a nominal state $x^*(t)$, as well as a state/control $x^*(t)$, $u^*(t)$ on the nominal state/control trajectory that is being tracked $x^*: [0, T] \rightarrow \X$, $u^*: [0, T] \rightarrow \U$. We refer to the LHSs of \eqref{eq:strong_con} and \eqref{eq:weak} as $C^s(x)$ and $C^w(\tilde x,x^*,u^*)$, respectively. Intuitively, \eqref{eq:strong_con} is a contraction condition simplified by the orthogonality condition \eqref{eq:strong_orth}, which together imply that all directions where the differential dynamics lack controllability must be naturally contracting at rate $\lambda$. The conditions \eqref{eq:strong} are stronger than \eqref{eq:weak}, which does not make this orthogonality assumption. 

How do we recover a tracking feedback controller $u(\tilde x,x^*,u^*)$ for \eqref{eq:g} from \eqref{eq:strong} and \eqref{eq:weak}? For \eqref{eq:strong}, the controller is implicit in the dual metric $W(x)$, and can be computed by solving a nonlinear optimization problem, which can be solved at runtime with pseudospectral methods \cite{sumeet_icra, leung}. In \eqref{eq:weak}, $u(\tilde x,x^*,u^*)$ is directly involved as the function defining $K$; as a consequence, $M(x)$ and $u(\tilde x,x^*,u^*)$ both need to be found. Thus for our purposes, the benefit of using \eqref{eq:strong} is that we have fewer parameters to learn. However, as some systems may not satisfy the properties needed to apply \eqref{eq:strong}, we resort to using \eqref{eq:weak} in these cases (see Sec. \ref{sec:method_learning_weak}). Finally, for a given CCM $M(x)$ and associated controller $u(\tilde x,x^*,u^*)$, for an unperturbed system tracking a nominal trajectory $x^*(t)$, the Riemannian energy $\energy(x^*(t), x(t))$ satisfies $\Vert x(t) - x^*(t) \Vert \le \beta \Vert x(0) - x^*(0)\Vert e^{-\lambda t}$ for an overshoot constant $\beta$, and thus the Euclidean distance also decays at this rate. If the system is subjected to bounded perturbations, it is instead guaranteed to remain in a tube around $x^*(t)$.

\subsection{Problem statement}

Our method has three major components. First, we learn a model \eqref{eq:g}, and then learn a contraction metric $M(x)$ and/or controller $u(\tilde x,x^*,u^*)$ for \eqref{eq:g}. Next, we analyze the learned \eqref{eq:g}, $M(x)$, and/or $u(\tilde x,x^*,u^*)$ to determine a trusted domain $D \subseteq \mathcal{X} \times \mathcal{U}$ where trajectories can be robustly tracked. Finally, we design a planner which steers between states in $D$, such that under the tracking controller $u(\tilde x, x^*, u^*)$, the system remains safe in execution and reaches the goal.
In this paper, we represent the approximate dynamics $\g(x,u)$ with an NN, though our method is agnostic to the structure of the model and how it is derived. 
Let $\S = \{(x_i,u_i,\f(x_i,u_i))\}_{i=1}^N$ be the training data for $g$ obtained by any means (i.e. random sampling, expert demonstrations, etc.), and let $\Psi = \{(x_j,u_j,\f(x_j,u_j))\}_{j=1}^M$ be a set of independent and identically distributed (i.i.d.) samples collected near $\S$. Then, our method involves solving the following:

\begin{problem}[Learning]\label{prob:learn}
Given $\S$, learn a control-affine model $\g$, a contraction metric $M(x)$, and find a contraction-based controller $u(\tilde x,x^*,u^*)$ that satisfies \eqref{eq:strong} or \eqref{eq:weak} over $\S$.
\end{problem}

\begin{problem}[Analysis]\label{prob:D}
	Given $\Psi$, $\g$, $M(x)$, and $u(\tilde x,x^*,u^*)$, design a trusted domain $D$. In $D$, find a model error bound $\Vert \f(x,u) - \g(x,u)\Vert \le e(x,u)$, for all $(x,u) \in D$, and verify if for all $x\in D$, $M$ and $u$ are valid, i.e. satisfying \eqref{eq:strong}/\eqref{eq:weak}.
\end{problem}

\begin{problem}[Planning]\label{prob:path}
Given $\g$, $M(x)$, $u(\tilde x,x^*,u^*)$, start $x_I$, goal $x_G$, goal tolerance $\mu$, maximum tracking error tolerance $\hat\mu$, trusted domain $D$, and $\normalfont\Xsafe$, plan a nominal trajectory $x^*: [0, T] \rightarrow \X$, $u^*: [0, T] \rightarrow \U$ under the learned dynamics $\g$ such that $x(0) = x_I$, $\dot x = g(x,u)$, $\|x(T) - x_G\| \leq \mu$, and $x(t)$, $u(t)$ remains in $D \cap \normalfont\Xsafe$ for all $ t \in [0, T]$. Also, guarantee that in tracking $(x^*(t), u^*(t))$ under the true dynamics $\f$ with $u(\tilde x, x^*, u^*)$, the system remains in $D \cap \normalfont\Xsafe$ and reaches $\B_{\hat\mu+\mu}(x_G)$.
\end{problem}

\section{Method}

We first describe a spatially-varying, Lipschitz constant-based model error bound, and derive a trajectory tracking error bound for a CCM-based controller under this error description (Sec. \ref{sec:method_bounds}). We then show how we can learn a dynamics model, CCM, and tracking controller to optimize the tracking error bound (Sec. \ref{sec:method_learning}). Then we show how we can design a trusted domain $D$ and verify the correctness of the model error bound, tracking error bound, and controller within $D$ (Sec. \ref{sec:method_verification}). Finally, we show how this tracking bound can be embedded into a sampling-based planner to ensure that plans provably remain safe in execution and reach the goal (Sec. \ref{sec:method_planning}). We summarize our method in Fig. \ref{fig:flowchart}.

\subsection{CCM-based tracking tubes under Lipschitz model error}\label{sec:method_bounds}

We first establish a spatially-varying bound on model error within a trusted domain $D$ which can be estimated from the model error evaluated at training points. For a single training point $(\tx,\tu)$ and a novel point $(x,u)$, we can bound the error between the true and learned dynamics at $(x,u)$ using the triangle inequality and Lipschitz constant of the error $L_{\f-\g}$:
\begin{equation}\label{eq:lip_bound_single}
\begin{aligned}
    \|\f(&x,u) - \g(x,u)\| \\
    &\leq L_{\f-\g} \|(x,u) - (\tx,\tu)\| + \|\f(\tx,\tu) - \g(\tx,\tu)\|.
\end{aligned}
\end{equation}

As this holds between the novel point and all training points, the following (possibly) tighter bound can be applied:
\begin{equation}\label{eq:lip_bound_multi}\small
\begin{aligned}
	\|\f(x,u) - \g(x,u)\| \leq  \min_{1 \le i \le N}&\Big\{ L_{\f-\g}\|(x,u) - (x_i,u_i)\| \\[-8pt]& + \|\f(x_i,u_i) - \g(x_i,u_i)\|\Big\}.
\end{aligned}
\end{equation}

To exploit higher model accuracy near the training data, we define $D$ as the union of $r$-balls around $\S$, where $r < \infty$:

\begin{equation}\label{eq:def_D}
    D = \bigcup_{i=1}^N \, \B_{r}(x_i,u_i).
\end{equation}

For these bounds to hold, $L_{\f-\g}$ must be a valid Lipschitz constant over $D$. In Sec. \ref{sec:method_verification}, we discuss how to obtain a probabilistically-valid estimate of $L_{\f-\g}$ and how to choose $r$. We now  derive an upper bound $\bar\epsilon(t)$ on the Euclidean tracking error $\epsilon(t)$ around a nominal trajectory $(x^*(t), u^*(t)) \subseteq D$ for a given metric $M(x)$ and feedback controller $u(\tilde x, x^*, u^*)$, such that the executed and nominal trajectories $x(t)$ and $x^*(t)$ satisfy $\Vert x(t) - x^*(t) \Vert \le \bar\epsilon(t)$, for all $t \in [0, T]$, when subject to the model error description \eqref{eq:lip_bound_multi}. In Sec. \ref{sec:method_learning}, we discuss how $M$ and $u$ can be learned from data.

In \cite{sumeet_icra}, it is shown that by using a controller which is contracting with rate $\lambda $ according to metric $M(x)$ for the \textit{nominal} dynamics \eqref{eq:g}, the Riemannian energy $\energy(t)$ of a \textit{perturbed} control-affine system $\dot x(t) = \gzero(x(t)) + \gone(x(t))u(t) + d(t)$ is bounded by the following differential inequality:
\begin{equation}\label{eq:energy_ineq}
	D^+\energy(t) \le -2\lambda \energy(t) + 2 \sqrt{\energy(t) \bar\lambda_D(M)}\Vert d(t) \Vert,
\end{equation}

\noindent where $\bar\lambda_D(M) = \sup_{x \in D} \bar\lambda(M(x))$ and $D^+(\cdot)$ is the upper Dini derivative of $(\cdot)$. Here, the energy $\energy(t) = \energy(x^*(t), x(t))$ is the squared trajectory tracking error according to the metric $M(x)$ at a given time $t$, and $d(t)$ is an external disturbance. Suppose that the only disturbance to the system comes from the discrepancy between the learned and true dynamics, i.e. $d(t) = \f(x(t), u(t)) - \g(x(t),u(t))$\footnote{In addition to model error, we can also handle \textit{runtime} external disturbances with a known upper bound; we assume the training data is noiseless.}. For short, let $e_i \doteq \Vert \f(x_i, u_i) - \g(x_i, u_i) \Vert$ be the training error of the $i$th data-point. In this case, we can use \eqref{eq:lip_bound_multi} to write:
\begin{equation}\label{eq:dist_bound_implicit}
	\Vert d(t) \Vert \le \min_{1\le i\le N} \Bigg\{L_{\f-\g} \Bigg\Vert \begin{bmatrix}x(t)\\ u(t)\end{bmatrix} -  \begin{bmatrix}x_i\\ u_i\end{bmatrix} \Bigg\Vert + e_i\Bigg\}.
\end{equation}
As \eqref{eq:dist_bound_implicit} is spatially-varying, it suggests that in solving Prob. \ref{prob:path}, plans should stay near low-error regions to encourage low error in execution. However, \eqref{eq:dist_bound_implicit} is only implicit in the plan, depending on the state visited and feedback control applied \textit{in execution}: $x(t) = x^*(t) + \tilde x(t)$ and $u(\tilde x(t),x^*(t),u^*(t)) = u^*(t) + u_\fb(t)$. To derive a tracking bound that can directly inform planning, we first introduce the following lemma:

\begin{lem}
	The Riemannian energy $\energy(t)$ of the perturbed system $\dot x(t) = f(x(t)) + B(x(t))u(t) + d(t)$, where $\Vert d(t) \Vert$ satisfies \eqref{eq:dist_bound_implicit}, satisfies the differential inequality \eqref{eq:bigdiffeq}, where $\underline\lambda_D(M) = \inf_{x \in D} \underline\lambda(M(x))$, $\normalfont\bar{u}_\fb(t)$ is a time-varying upper bound on the feedback control $\Vert u(t) - u^*(t)\Vert$, and $i^*(t)$ achieves the minimum in \eqref{eq:dist_bound_implicit}.
\end{lem}
\begin{proof}[Proof sketch]
We use the triangle inequality to simplify \eqref{eq:dist_bound_implicit}:
\begin{equation*}\footnotesize
\begin{split}
	\Vert d(t) \Vert \le \min_{1\le i\le N} \Bigg\{L_{\f-\g}\Bigg( \Bigg\Vert \begin{bmatrix}x^*(t)\\ u^*(t)\end{bmatrix} -  \begin{bmatrix}x_i\\ u_i\end{bmatrix} \Bigg\Vert + \Bigg\Vert \begin{bmatrix}\tilde x(t)\\ u_\fb(t)\end{bmatrix} \Bigg\Vert\Bigg) + e_i\Bigg\} \\
	\le L_{\f-\g}\Bigg\Vert \begin{bmatrix}\tilde x(t)\\ u_\fb(t)\end{bmatrix} \Bigg\Vert + \min_{1\le i\le N} \Bigg\{L_{\f-\g} \Bigg\Vert \begin{bmatrix}x^*(t)\\ u^*(t)\end{bmatrix} -  \begin{bmatrix}x_i\\ u_i\end{bmatrix} \Bigg\Vert + e_i\Bigg\}.
\end{split}
\end{equation*}

\noindent Note that as $\Vert d(t)\Vert$ depends on $\tilde x(t)$, the disturbance bound itself depends on $\epsilon(t)$. To make this explicit, we use $\Vert \tilde x(t)\Vert = \epsilon(t)$ and $\Vert u_\fb(t) \Vert \le \bar{u}_\fb(t)$ to obtain 
\begin{equation}\label{eq:disturbance_bnd}
\hspace{-4pt}\begin{array}{>{\displaystyle}c >{\displaystyle}l}
	\Vert d(t) \Vert \le & L_{\f-\g}\big(\epsilon(t) + \bar{u}_\fb(t)\big) + \\
	& \min_{1\le i\le N} \Bigg\{L_{\f-\g} \Bigg\Vert \begin{bmatrix}x^*(t)\\ u^*(t)\end{bmatrix} -  \begin{bmatrix}x_i\\ u_i\end{bmatrix} \Bigg\Vert + e_i\Bigg\}.
\end{array}
\end{equation}
To obtain $\bar{u}_\fb(t)$, if we use CCM conditions \eqref{eq:strong}, we can use the optimization-based controller in \cite{sumeet_icra} (c.f. Sec. \ref{sec:prelim_ccm}), which admits the upper bound \cite[p.28]{sumeet_icra}:
\begin{equation}\label{eq:ctrl_bound_geodesic}\small
	\Vert u_\fb(t) \Vert \le \epsilon(t) \sup_{x \in D} \frac{\bar\lambda(L(x)^{-\top} F(x) L(x)^{-1})}{2\underline\sigma_{>0}(\gone^\top(x) L(x)^{-1})}\doteq \epsilon(t)\delta_u,
\end{equation}
where $W(x) = L(x)^\top L(x)$, $F(x) = -\partial_{\gzero} W(x) + \reallywidehat{\frac{\partial \gzero(x)}{\partial x} W(x)} + 2\lambda W(x)$, and $\underline\sigma_{>0}(\cdot)$ is the smallest positive singular value. If we instead use condition \eqref{eq:weak}, we must estimate $\bar{u}_\fb(t)$ for the learned controller (c.f. Sec. \ref{sec:method_verification}).

To obtain the result, we plug \eqref{eq:disturbance_bnd} into \eqref{eq:energy_ineq} after relating $\epsilon(t)$ with $\energy(t)$. Since $\energy(x^*(t), x(t)) = \textrm{dist}^2(x^*(t), x(t)) \ge \underline\lambda_D(M)\Vert x^*(t) - x(t) \Vert^2$, we have that $\epsilon(t) \le \sqrt{\energy(t)/\underline\lambda_D(M)}$. Finally, we can plug all of these components into \eqref{eq:energy_ineq} to obtain \eqref{eq:bigdiffeq}, where $i^*(t)$ denotes a minimizer of \eqref{eq:dist_bound_implicit}. 
\end{proof}

\begin{figure*}
	\begin{equation}\label{eq:bigdiffeq}
	D^+\energy(t) \le -2\Bigg(\lambda -L_{\f-\g}\sqrt{\frac{\bar\lambda_D(M)}{\underline\lambda_D(M)}}\Bigg)\energy(t) + 2 \sqrt{\energy(t)\bar\lambda_D(M)}\Bigg(L_{\f-\g}\Bigg( \Bigg\Vert \begin{bmatrix}x^*(t)\\ u^*(t)\end{bmatrix} -  \begin{bmatrix}x_{i^*(t)}\\ u_{i^*(t)}\end{bmatrix} \Bigg\Vert + \bar{u}_\fb(t)\Bigg) +e_{i^*(t)}\Bigg)
	\end{equation}\vspace{-5pt}
\end{figure*}

For intuition, let us interpret the spatially-varying disturbance bound \eqref{eq:disturbance_bnd}. Note that \eqref{eq:disturbance_bnd} depends on several components. First, it depends on $\epsilon(t)$, which in turn relies on the disturbance magnitude: intuitively, this is because with better tracking performance, the system will visit a smaller set of states and thus experience lower worst-case model error. Second, it depends on $u_\fb(t)$: if a large feedback is applied, the combined control input $u(t) = u^*(t) + u_\fb(t)$ can be far from the control inputs that the learned model is trained on, possibly leading to high error. Finally, it is driven by the model error via closeness to the training data and the corresponding training error (minimization term of \eqref{eq:disturbance_bnd}). We can also gain some insight by comparing our derived tracking error bound \eqref{eq:bigdiffeq} with the tracking bound for a uniform disturbance description \eqref{eq:energy_ineq}. Notice that the ``effective" contraction rate $\lambda - L_{\f-\g}\sqrt{\frac{\bar\lambda_D(M)}{\underline\lambda_D(M)}}$ shrinks with $L_{\f-\g}$, as the model error grows with tracking error. If the optimization-based controller \cite{sumeet_icra} is used, the $\epsilon(t)$ dependence of \eqref{eq:ctrl_bound_geodesic} reduces this rate to $\lambda - L_{\f-\g}\sqrt{\frac{\bar\lambda_D(M)}{\underline\lambda_D(M)}}(1+\delta_u)$. Now, we are ready to obtain the tracking bound:

\begin{theorem}[Tracking bound under \eqref{eq:dist_bound_implicit}]\label{thm:trk_bnd}
	Let $\energy_\textrm{RHS}$ denote the RHS of \eqref{eq:bigdiffeq}. Assuming that the perturbed system $\dot x(t) = f(x(t)) + B(x(t))u(t) + d(t)$ satisfies $\energy(t_1) \le \energy_{t_1}$ and $\Vert d(t) \Vert$ satisfies \eqref{eq:dist_bound_implicit}. Then, $\bar\epsilon(t)$ is described at some $t_2 > t_1$ as:
	\begin{equation}\label{eq:trk_bnd}\normalfont
		\hspace{-5pt}\bar\epsilon(t_2) = \sqrt{\big(\textstyle\int_{\tau=t_1}^{t_2} \energy_\textrm{RHS}(t) d\tau\big)/\underline\lambda_D(M)}, \quad \energy(t_1) = \energy_{t_1}.
	\end{equation}
\end{theorem}
\begin{proof}[Proof sketch]
We apply the Comparison Lemma \cite{khalil} on \eqref{eq:bigdiffeq}. To use the Comparison Lemma, the right hand side of \eqref{eq:bigdiffeq} must be Lipschitz continuous in $\energy$ and continuous in $t$ \cite{khalil}. Lipschitz continuity in $\energy$ holds for \eqref{eq:bigdiffeq} since it only contains a linear and a square-root term involving $\energy$, each with finite coefficients; continuity in $t$ follows by noting that the minimization in \eqref{eq:disturbance_bnd} is the pointwise minimum of $N$ continuous functions of $t$, and thus is itself a continuous function of $t$.
\end{proof}
	Note that Thm. \ref{thm:trk_bnd} provides a Euclidean tracking error tube under the model error bound \eqref{eq:dist_bound_implicit} for \textit{any} nominal trajectory. Moreover, as \eqref{eq:trk_bnd} can be integrated incrementally in time, it is well-suited to guide planning in an RRT (Rapidly-exploring Random Tree \cite{lavalle2001randomized}); see Sec. \ref{sec:method_planning} for more details.

\subsection{Optimizing CCMs and controllers for the learned model}\label{sec:method_learning}

Having derived the tracking error bound, we discuss our solution to Prob. \ref{prob:learn}, i.e. how we learn the control-affine dynamics \eqref{eq:g}, a contraction metric $M(x)$, and (possibly) a stabilizing controller $u$ in a way that optimizes the size of \eqref{eq:trk_bnd}. In this paper, we represent $\gzero(x)$, $\gone(x)$, $M(x)$, and $u(\tilde x, x^*, u^*)$ as deep neural networks. 

Ideally, we would learn the dynamics jointly with the contraction metric to minimize the size of the tracking tubes \eqref{eq:trk_bnd}. However, we observe this leads to poor learning, generally converging to a valid CCM for highly inaccurate dynamics. Instead, we elect to use a simple two step procedure: we first learn $\g$, and then fix $\g$ and learn $M(x)$ and $u(\tilde x,x^*,u^*)$ for that model. While this is sufficient for our examples, in general an alternation procedure may be helpful.

\noindent \textbf{Dynamics learning}. Inspecting \eqref{eq:bigdiffeq}, we note that the model-error related terms are the Lipschitz constant $L_{\f-\g}$ and training error $e_i$, $i=1,\ldots,N$. Thus, in training the dynamics, we use a loss function penalizing the mean squared error and a batch-wise estimate of the Lipschitz constant:
\begin{equation}\small\label{eq:loss_dyn}
	L_\textrm{dyn} = \frac{1}{N_b}\sum_{i=1}^{N_b} e_i^2 + \alpha_1 \max_{1 \le i,j \le N_b}\Bigg\{ \frac{\Vert e_i - e_j \Vert}{\Vert (x_i,u_i) - (x_j,u_j)\Vert}\Bigg\},
\end{equation}
where $e_i = \Vert \g(x_i, u_i) - \f(x_i, u_i) \Vert$, $N_b \le N$ is the batch size, and $\alpha_1$ trades off the objectives. Note that \eqref{eq:loss_dyn} promotes $e_i$ to be small while remaining smooth over the training data, in order to encourage similar properties to hold over $D$.

\noindent \textbf{CCM learning}. We describe two variants of our learning approach, depending on if the stronger CCM conditions \eqref{eq:strong_con} and \eqref{eq:strong_orth} or the weaker condition \eqref{eq:weak} is used.

\subsubsection{Using \eqref{eq:strong_con} and \eqref{eq:strong_orth}}\label{sec:method_learning_strong}
We parameterize the dual metric as $W(x) = W_{\theta_w}(x)^\top W_{\theta_w}(x) + \underline{w} \I_{n\times n}$, where $W_{\theta_w}(x) \in \mathbb{R}^{n_x \times n_x}$, $\theta_w$ are the NN weights, and $\underline{w}$ is a minimum eigenvalue hyperparameter. This structure ensures that $W(x)$ is symmetric positive definite for all $x$. To enforce \eqref{eq:strong_con}, we follow \cite{dawei}, relaxing the matrix inequality to an unconstrained penalty $L_\textrm{NSD}^s$ over training data, where:

\begin{equation}\label{eq:loss_nsd}
	L_\textrm{NSD}^{(\cdot)} = \max_{1\le i\le N_b} \bar\lambda\big( C^{(\cdot)}(x_i) \big).
\end{equation}

As we ultimately wish \eqref{eq:strong_con} to hold everywhere in $D$, we can use the continuity in $x$ of the maximum eigenvalue $\bar\lambda(\lambda^s(x))$ to verify if \eqref{eq:strong_con} holds over $D$ (c.f. Sec. \ref{sec:method_verification}). However, the equality constraints \eqref{eq:strong_orth} are problematic; by using unconstrained optimization, it is difficult to even satisfy \eqref{eq:strong_orth} on the training data, let alone on $D$. To address this, we follow \cite{sumeet_wafr} by restricting the dynamics learning to sparse-structured $\gone(x)$ of the form, where $\theta_B$ are NN parameters:
\begin{equation}\label{eq:sparse}
	\gone(x) = [\mathbf{0}_{n_x - n_u \times n_u}^\top,\ \ \gone_{\theta_{\gone}}(x)^\top]^\top.
\end{equation}
Restricting $\gone(x)$ to this form implies that to satisfy \eqref{eq:strong_orth}, $W(x)$ must be a function of only the first $n_x - n_u$ states \cite{sumeet_wafr}, which can be satisfied by construction. When this structural assumption does not hold, we use the method in Sec. \ref{sec:method_learning_weak}

In addition to the CCM feasibility conditions, we introduce novel losses to optimize the tracking tube size \eqref{eq:trk_bnd}. As \eqref{eq:trk_bnd} depends on the nominal trajectory, it is hard to optimize a tight upper bound on the tracking error independent of the plan. Instead, we maximize the effective contraction rate,
\begin{equation}\label{eq:loss_opt}\small
	L_\textrm{opt}^s = \alpha_2 \max_{1\le i\le N_b}\Big(\lambda - L_{\f-\g}\sqrt{\frac{\bar\lambda(M(x_i))}{\underline\lambda(M(x_i))}}(1+\delta_u(\tilde x_i))\Big),
\end{equation}
where $\alpha_2$ is a tuned parameter. Optimizing \eqref{eq:loss_opt} while ensuring \eqref{eq:strong_con} holds over the dataset is a challenging task for unconstrained NN optimizers. To ameliorate this, we use a linear penalty on constraint violation and switch to a logarithmic barrier \cite{boyd} to maintain feasibility upon achieving it. Let the combination of the logarithmic barrier and the linear penalty be denoted $\logb(\cdot)$. 
Then, the full loss function can be written as $\logb(-L_\textrm{NSD}^s) + L_\textrm{opt}^s$.

\subsubsection{Using \eqref{eq:weak}}\label{sec:method_learning_weak}

For systems that do not satisfy \eqref{eq:sparse}, we must use the weaker contraction conditions \eqref{eq:weak}. In this case, we cannot use the optimization-based controllers proposed in \cite{sumeet_icra}, and we instead learn $u(\tilde x,x^*,u^*)$ in tandem with $M(x)$. As in \eqref{eq:loss_nsd}, we enforce \eqref{eq:weak} by relaxing it to $L_\textrm{NSD}^w$.
We represent $u(\tilde x, x^*, u^*)$ with the following structure:
\begin{equation}\label{eq:tanh_controller}
	u(\tilde x, x^*, u^*) = |\theta_1^u|\tanh\big(u_{\theta_2^u}(\tilde x, x^*) \tilde x\big) + u^*,
\end{equation}
where $\theta_i^u$ are NN weights. The structure of \eqref{eq:tanh_controller} simplifies $\bar{u}_\fb$ estimation, since $\Vert u(\tilde x, x^*, u^*) - u^*\Vert < |\theta_1^u|$ for all $x$, $x^*$, $u^*$. We define $L_\textrm{opt}^w$ as in \eqref{eq:loss_opt}, without the $\delta_u$ term. Then, our full loss function is $\logb(-L_\textrm{NSD}^w) + L_\textrm{opt}^w + \alpha_3 |\theta_1^u|$.

We note that since the optimizer can reach a local minima, we may not find a valid CCM even if one exists. Some strategies we found to improve training reliability were the log-barrier and by gradually increasing $\alpha_2$ and $\alpha_3$ with the training epoch. If we can find a valid CCM on $\S$, we can verify if it is also valid over $D$, as we now discuss.

\subsection{Designing and verifying the trusted domain}\label{sec:method_verification}

The validity of the tracking bound \eqref{eq:trk_bnd} depends on having overestimates of $L_{\f-\g}$, $\bar\lambda_D(M)$, and $\delta_u$, an underestimate of $\underline \lambda_D(M)$, and the validity of \eqref{eq:strong_con}/\eqref{eq:weak} over $D$. In this section, we describe our solution to Prob. \ref{prob:D}, showing how to design a $D$ and how to estimate these constants over $D$.

First, let us consider how to estimate the constants over a given $D$. To obtain probabilistic over/under-estimates of these constants that are each valid with a user-defined probability $\rho$, we use a stochastic approach from extreme value theory. For brevity, we describe the basics and refer to \cite[p.3]{lipschitz_ral}, \cite{weng2018evaluating} for details. This approach estimates the maximum of a function $\eta(z)$ over a domain $\mathcal{Z}$ by collecting $N_s$ batches of i.i.d. samples of $z\in\mathcal{Z}$ of size $N_b$, and evaluating $\{s_j\}_{j=1}^{N_s} \doteq \{\max_{1\le i\le N_b}\eta(z_i^j)\}_{j=1}^{N_s}$ to obtain $N_s$ samples of empirical maxima. If the true maximum is finite and the distribution of sampled $s_j$ converges with increasing $N_s$, the Fisher-Tippett-Gnedenko (FTG) theorem \cite{de2007extreme} dictates that the samples converge to a Weibull distribution. This can be empirically verified by fitting a Weibull distribution to the $s_j$, and validating the quality of the fit with a Kolmogorov-Smirnov (KS) goodness-of-fit test \cite{degroot2013probability}. If this KS test passes, the location parameter of the fit Weibull distribution, adjusted with a confidence interval which scales in size with the value of the user-defined probability $\rho$, can serve as an estimate for the maximum which overestimates the true maximum with probability $\rho$.

To estimate $L_{\f-\g}$, we follow \cite{lipschitz_ral} to obtain a probabilistic overestimate of $L_{\f-\g}$ by defining $\mathcal{Z} = D$ and $\eta$ to be the slopes between pairs of points drawn i.i.d. from $D$. This approach can also be used as follows to estimate $\bar\lambda_D(M)$, $-\underline\lambda_D(M)$, and $\delta_u$. Since the eigenvalues of a continuously parameterized matrix function are continuous in the parameter \cite{lax07} (here, the parameter is $x$) and $D$ is bounded, these constants are finite, so by the FTG theorem, we can expect the samples $s_j$ to be Weibull. Hence, we can estimate these constants by defining $\mathcal{Z} = \proj_x(D)$, where $\proj_x(D) \doteq \bigcup_{\tx \in \S} \mathcal{B}_r(\tx) \supset \{x \mid \exists u, (x,u) \in D\}$, and by setting $\eta$ appropriately for each constant. Finally, FTG can also verify \eqref{eq:strong_con} and \eqref{eq:weak}, since the verification is equivalent to ensuring $\sup_{x\in \proj_x(D)} \bar\lambda(C^{(\cdot)}(x)) \le \lambda_\textrm{CCM}^{(\cdot)}$ for some $\lambda_\textrm{CCM}^{(\cdot)} < 0$. $\lambda_\textrm{CCM}^{s}$ can be estimated by setting $\mathcal{Z} = \proj_x(D)$ and $\eta(x) = \bar\lambda(C^s(x))$. To estimate $\lambda_\textrm{CCM}^{w}$, we set $\mathcal{Z} = \B_{\epsilon_\textrm{max}}(0) \times D$ and $\eta(\tilde x,x^*,u^*) = \bar\lambda(C^w(\tilde x,x^*,u^*))$, and sample $(x^*,u^*) \in D$ and $\tilde x \in \B_{\epsilon_\textrm{max}}(0)$. Here, $\epsilon_\textrm{max} \le \hat\mu$ will upper-bound the allowable tracking tube size during planning (c.f. Alg. \ref{alg:rrt}, line 8); thus, to ensure that planning is minimally constrained, $\epsilon_\textrm{max}$ should be selected to be as large as possible while maintaining $\lambda_\textrm{CCM}^{w} < 0$. As all samples are i.i.d., the probability of \eqref{eq:trk_bnd} holding, and thus the overall safety probability assured by our method, is the product of the user-selected $\rho$ for each constants. 

Before moving on, we note that other than for $L_{\f-\g}$, the estimation procedure does not affect data-efficiency, as it queries the \textit{learned dynamics} and requires no new data of the form $(x,u,\f(x,u))$. Moreover, some methods \cite{fazlyab2019efficient, JordanD20} deterministically give guaranteed upper bounds on the Lipschitz constant of NNs, and can be used to estimate all constants except $L_{\f-\g}$. We do not use these methods due to their scalability issues, but as further work is done in this area, these methods may also become applicable.

\begin{figure}
    \centering
    \includegraphics[width=\linewidth]{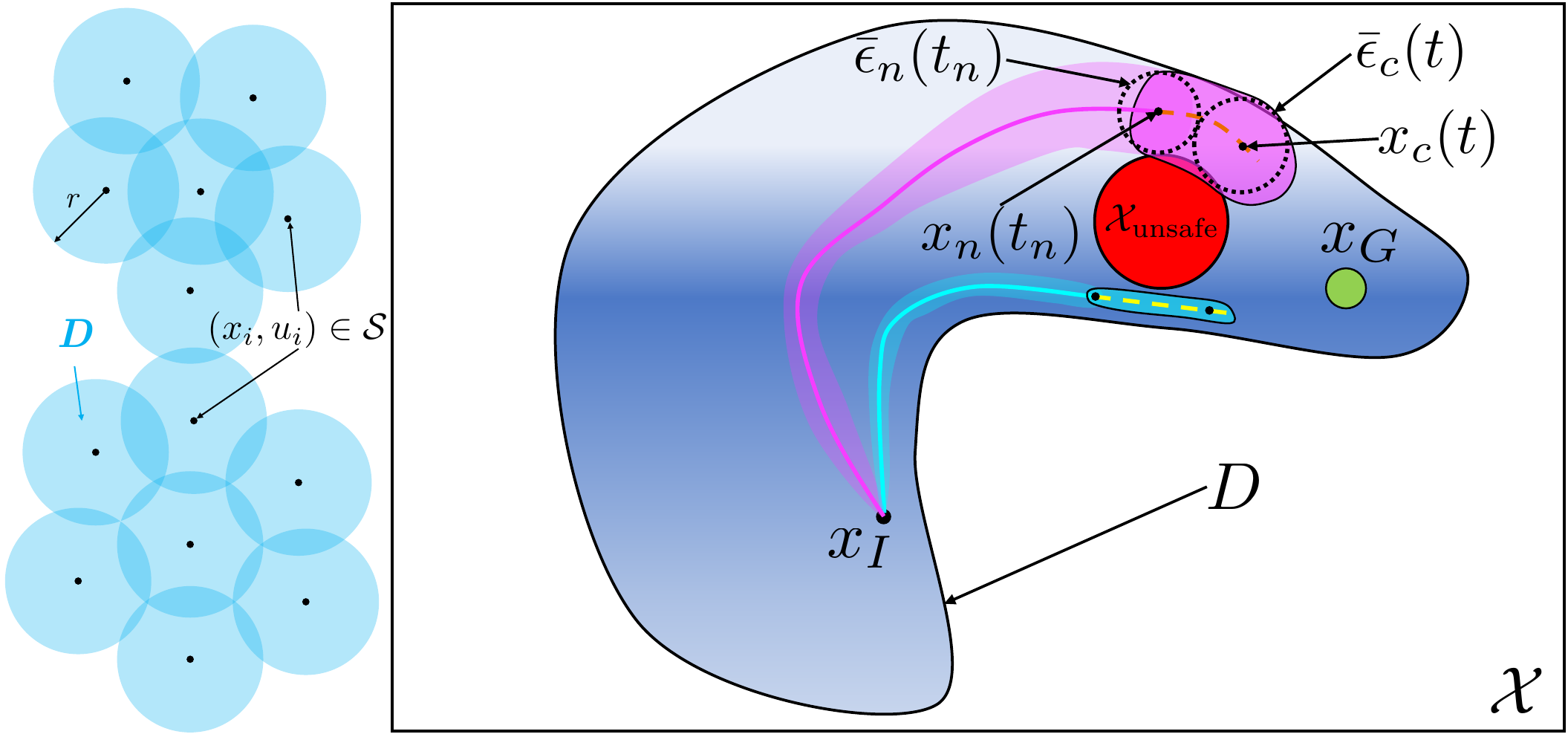}\vspace{-4pt}
    \caption{\textbf{Left}: an example of the trusted domain $D$, with the dataset $\S$ being shown as black dots. Note that a careful choice of $r$ is needed; for a slightly smaller $r$ than that shown in the figure, the upper and lower portions of $D$ will become disconnected, leading to plan infeasibility. \textbf{Right}: An example of LMTCD-RRT in action. Regions in $D$ that are shaded darker blue have smaller model error $\Vert d \Vert$; lighter shades have higher error $\Vert d \Vert$. The orange extension to the pink branch of the RRT is rejected, since the tube around that extension (dark magenta) exits $D$ and intersects with the obstacle; the larger size of this tube results from the pink branch traveling through higher error regions. In contrast, the cyan branch (lower) accepts the yellow candidate extension, as its corresponding tube (dark cyan) remains inside $D$ and collision-free; the smaller tube sizes reflect that the blue branch has traveled through lower-error regions. This type of behavior biases the planner to ultimately return a path that travels through lower-error regions.} \label{fig:tube_plan}
\end{figure}

Finally, we discuss how to select $r$, which determines $D$ (Fig. \ref{fig:tube_plan}, left). A reasonable choice of $r$ is one that is maximally permissive for planning, during which we will need to ensure that the tracking tube around the planned trajectory remains entirely in $D$ (c.f. Sec. \ref{sec:method_planning} for more details). However, finding this $r$ is non-trivial to achieve and requires trading off many factors. For large $r$, planning may become easier since this increases the size of $D$; however, model error and $L_{\f-\g}$ also degrades with increased $r$, which may make $\bar\epsilon(t)$ and $\bar{u}_\fb(t)$ grow, which in turn grows the tracking tube size, making it harder to fit the tube within $D$. Also, \eqref{eq:strong_con}/\eqref{eq:weak} may not be satisfied over $D$ for large $r$. For small $r$, the model error and $L_{\f-\g}$ remain smaller due to the closeness to $\S$, leading to smaller tubes, but planning can be challenging, as $D$ may be too small to contain even these smaller tubes. In particular, planning between two states in $D$ can become infeasible if $D$ becomes disconnected.

\begin{algorithm*}\label{alg:rrt}
\KwIn{$x_I$, $x_G$, $\S$, $\{e_i\}_{i=1}^N$, estimated constants, $\mu$, $\energy_0$}
\DontPrintSemicolon
\SetKwFunction{SampleState}{SampleState}
\SetKwFunction{SampleNode}{SampleNode}
\SetKwFunction{SampleCandidateControl}{SampleCandidateControl}
\SetKwFunction{NearestNeighbor}{NN}
\SetKwFunction{IntegrateLearnedDyn}{IntegrateLearnedDyn}
\SetKwFunction{DCheck}{DCheck}
\SetKwFunction{TrkErrBnd}{TrkErrBndEq12}
\SetKwFunction{SteerInDEpsilon}{SteerInDEpsilon}
\SetKwFunction{OneStepReachable}{OneStep}
\SetKwFunction{Model}{Model}
\SetKwFunction{ConstructPath}{ConstructPath}
\SetKwFunction{InCollision}{InCollision}

$\T \leftarrow \{(x_I, \sqrt{\energy_0/\underline\lambda_D(M)}, 0)\}; \mathcal{P} \leftarrow \{(\emptyset, \emptyset)\}$ \tcp*{node: state, energy, time; parent: previous control/dwell time}

\While{\upshape True}{
	$(x_{\textrm{n}}, \bar\epsilon_\textrm{n}, t_\textrm{n}) \leftarrow $ \SampleNode{$\T$} \tcp*{sample a node from the tree for expansion}
	$(u_\textrm{c}, t_\textrm{c}) \leftarrow $ \SampleCandidateControl() \tcp*{sample a control action and dwell time}
	$(x_\textrm{c}^*(t), u_\textrm{c}^*(t)) \leftarrow$ \IntegrateLearnedDyn($x_{\textrm{n}}$, $u_\textrm{c}$, $t_\textrm{c}$)\tcp*{apply control for dwell time; get candidate tree extension}
	$\bar\epsilon_c(t) \leftarrow$ \TrkErrBnd($\bar\epsilon_\textrm{n}$, $x_\textrm{c}^*(t)$, $u_\textrm{c}^*(t)$, $\S$, $\{e_i\}_{i=1}^N$)\tcp*{compute tracking error tube for candidate tree extension}
	$D_\textrm{chk}^1 \leftarrow (x^*(t), u^*(t)) \in D_{\bar\epsilon_c(t) - \bar u_\fb(t)}, \forall t \in [t_\textrm{n}, t_\textrm{n}+t_\textrm{c})$\tcp*{check if tube around cand. extension remains within $D$}
	\lIf{controller learned using \eqref{eq:weak}}{$D_\textrm{chk}^2 \leftarrow \bar\epsilon_c(t) \le \epsilon_\textrm{max}, \forall t \in [t_\textrm{n}, t_\textrm{n}+t_\textrm{c})$}
	\lElse{$D_\textrm{chk}^2 \leftarrow $ True\tcp*[f]{if using a controller satisfying \eqref{eq:weak}, check if accumulated tracking error is below tolerance}}
	$C \leftarrow $ \InCollision($x^*(t)$, $u^*(t)$, $\bar\epsilon_c(t)$)\tcp*{check if tracking error tube collides with obstacles}
	\lIf{$\normalfont D_\textrm{chk}^1 \wedge D_\textrm{chk}^2 \wedge \neg C$}{$\T \leftarrow \T \cup \{(x^*(t_\textrm{n}+t_\textrm{c}), \bar\epsilon_c(t_\textrm{n}+t_\textrm{c}), t_\textrm{c})\}$; $\mathcal{P} \leftarrow \mathcal{P} \cup \{(u_\textrm{c}, t_\textrm{c})\}$}
	\lElse{continue\tcp*[f]{add node and corresponding parent if all checks pass}}
	\lIf{$\exists t, x_c^*(t) \in \B_{\mu}(x_G)$}{break; return plan \tcp*[f]{return path upon reaching goal}}
}
\caption{\lmtd}
\end{algorithm*}

To trade off these competing factors, we propose the following solution for selecting $r$. We first find a minimum $r$, $r_\textrm{connect}$, such that $D$ is fully-connected. Depending on how $\S$ is collected, one may wish to first filter out outliers that lie far from the bulk of the data. We calculate the connected component by considering the dataset as a graph, where an edge between $(x_i,u_i), (x_j,u_j) \in \S$ exists if $\Vert (x_i,u_i) - (x_j,u_j)\Vert \le r$. We then determine if the contraction condition \eqref{eq:strong_con}/\eqref{eq:weak} is satisfied for $r = r_\textrm{connect}$, using the FTG-based procedure. If it is not satisfied, we decrement $r$ until \eqref{eq:strong_con}/\eqref{eq:weak} holds, and select $r$ as the largest value for which \eqref{eq:strong_con}/\eqref{eq:weak} are satisfied. Since $r < r_\textrm{connect}$ in this case, planning can only be feasible between start and goal states within each connected component; to rectify this, more data should be collected to train the CCM/controller. If the contraction condition is satisfied at $r = r_\textrm{connect}$, we incrementally increase $r$, starting from $r_\textrm{connect}$. In each iteration, we first determine if the contraction condition \eqref{eq:strong_con}/\eqref{eq:weak} is still satisfied for the current $r$, using the FTG-based procedure. If the contraction condition is satisfied, we evaluate an approximate measure of planning permissiveness under ``worst-case" conditions\footnote{Roughly, this compares the size of $D$ to the tracking error tube size and feedback control bound, c.f. Sec. \ref{sec:method_planning} and Thm. \ref{thm:deps} for further justification.}: $r - \bar\epsilon(t) - \bar u_\fb(t)$, evaluated at a fixed time $t = T_\textrm{query}$, where $\bar\epsilon(t)$ and $\bar u_\fb(t)$ are computed assuming that for all $t\in[0,T_\textrm{query}]$, $\Vert (x^*(t), u^*(t)) - (x_{i^*(t)}, u_{i^*(t)}))\Vert = \max_{1\le i \le N} \min_{1\le j \le N} \Vert (x_i,u_i) - (x_j,u_j)\Vert$, i.e. the dispersion of the training data, and experiences the worst training error (i.e. $e_{i^*(t)} = \max_{1\le i \le N} e_i$, for all $t$). If the contraction condition is not satisfied, we terminate the search and select the $r$ with the highest permissiveness, as measured by the aforementioned procedure.

\subsection{Planning with the learned model and metric}\label{sec:method_planning}

Finally, we discuss our solution to safely planning with the learned dynamics (Prob. \ref{prob:path}). We develop an incremental sampling-based planner akin to a kinodynamic RRT \cite{lavalle2001randomized}, growing a search tree $\T$ by forward-propagating sampled controls held for sampled dwell-times, until the goal is reached. To ensure the system remains within $D$ in execution (where the contraction condition and \eqref{eq:trk_bnd} are valid), we impose additional constraints on where $\T$ is allowed to grow.

Denote $D_{q} = D \ominus \mathcal{B}_{q}(0)$ as the state/controls which are at least distance $q$ from the complement of $D$, where $\ominus$ refers to the Minkowski difference. Since \eqref{eq:trk_bnd} defines tracking error tubes for \textit{any} given nominal trajectory, we can efficiently compute tracking tubes along any candidate edge of an RRT.  Specifically, suppose that we wish to extend the RRT from a state on the planning tree $x_\textrm{cand}^*(t_1)$ with initial energy satisfying $\energy_\textrm{cand}(t_1) \le \energy_{t_1}$ to a candidate state $x_\textrm{cand}^*(t_2)$ by applying control $u$ over $[t_1, t_2)$. This information is supplied to \eqref{eq:trk_bnd}, and we can obtain the tracking error $\bar\epsilon_\textrm{cand}(t)$, for all $t\in[t_1, t_2)$. Then, if we enforce that $(x^*(t),u^*(t)) \in D_{\bar\epsilon_\textrm{cand}(t) + \bar{u}_\fb(t)}$ for all $t \in [t_1, t_2)$, we can ensure that the true system remains within $D$ when tracked with a controller that satisfies $u_\fb(t) \le \bar u_\fb(t)$ in execution. Otherwise, the extension is rejected and the sampling continues. When using a learned $u(\tilde x,x^*,u^*)$ with \eqref{eq:weak}, an extra check that $\bar\epsilon_\textrm{cand}(t) \le \epsilon_\textrm{max}$ is needed to remain in $\B_{\epsilon_\textrm{max}}(0)\times D$ (c.f. Sec. \ref{sec:method_verification}). Since $D$ is a union of balls, exactly checking $(x^*(t),u^*(t)) \in D_{\bar\epsilon(t)+\bar{u}_\fb(t)}$ can be unwieldy. However, a conservative check can be efficiently performed by evaluating \eqref{eq:deps_check}:

\begin{theorem}\label{thm:deps}
If \eqref{eq:deps_check} holds for some index $1\le i \le N$ in $\S$,
\begin{equation}\label{eq:deps_check}\normalfont
	\Vert(x^*(t),u^*(t)) - (x_i,u_i)\Vert \le r - \bar\epsilon(t)-\bar{u}_\fb(t),
\end{equation}
then $\normalfont(x^*(t),u^*(t)) \in D_{\bar\epsilon(t)+\bar{u}_\fb(t)}$.
\end{theorem}
\begin{proof}
	By \eqref{eq:deps_check}, all $(\tx,\tu) \in \B_{\bar\epsilon(t)+\bar{u}_\fb(t)}(x^*(t),u^*(t))$ satisfy $\Vert(x_i,u_i) - (\tx,\tu)\Vert \le r$ by the triangle inequality. Thus, $\B_{\bar\epsilon(t)+\bar{u}_\fb(t)}(x^*(t),u^*(t)) \subset \B_{r}(x_i,u_i) \subset D$. As $\B_{\bar\epsilon(t)+\bar{u}_\fb(t)}(x^*(t),u^*(t)) \subset D$, $(x^*(t),u^*(t))$ is at least ${\bar\epsilon(t)+\bar{u}_\fb(t)}$ distance from the complement of $D$; thus, $(x^*(t),u^*(t)) \in D_{\bar\epsilon(t)+\bar{u}_\fb(t)}$.
\end{proof}

We perform collision checking between the tracking tubes and the obstacles, which we assume are expanded for the robot geometry; this is made easier since \eqref{eq:trk_bnd} defines a sphere for all time instants. We visualize our planner (Fig. \ref{fig:tube_plan}, right), which we denote \textbf{L}earned \textbf{M}odels in \textbf{T}rusted \textbf{C}ontracting \textbf{D}omains (LMTCD-RRT), and summarize it in Alg. \ref{alg:rrt}. We conclude with the following correctness result:\vspace{-5pt}

\begin{theorem}[\lmtd correctness]
	Assume that the estimated $L_{\f-\g}$, $\bar\lambda_D(M)$, $\normalfont\bar{u}_\fb(t)$, and $\normalfont\lambda_\textrm{CCM}^{(\cdot)}$ overapproximate their true values and the estimated $\underline\lambda_D(M)$ underapproximates its true value. Then, when using a controller $u(\tilde x,x^*,u^*)$ derived from \eqref{eq:strong}, Alg. \ref{alg:rrt} returns a trajectory $(x^*(t), u^*(t))$ that remains within $D$ in execution on the true system. Moreover, when using a controller $u(\tilde x,x^*,u^*)$ derived from \eqref{eq:weak}, Alg. \ref{alg:rrt} returns a trajectory $(x^*(t), u^*(t))$ such that $(\tilde x^*(t), x^*(t), u^*(t))$ remains in $\normalfont\B_{\epsilon_\textrm{max}}(0) \times D$ in execution on the true system.
\end{theorem}
\begin{proof}
	First, we consider a controller that is derived from \eqref{eq:strong}. By Thm. \ref{thm:deps}, Alg. \ref{alg:rrt} returns a plan $(x^*(t), u^*(t))$ such that for all $t \in [0, T]$, $(x^*(t), u^*(t)) \in D_{\epsilon(t)+\bar{u}_\fb(t)}$. If the constants are estimated correctly, Thm. \ref{thm:trk_bnd} holds, ensuring that the tracking error $\epsilon(t)$ in execution is less than $\bar\epsilon(t)$, for all $t\in[0,T]$. Furthermore, by the correct estimation of $\bar{u}_\fb(t)$, the feedback control satisfies $\Vert u_\fb(t) \Vert \le \bar u_\fb(t)$. Thus, for all $t \in [0, T]$, the state/control in execution $(x(t), u(t))$ remains in $\B_{\epsilon(t)+\bar{u}_\fb(t)}(x^*(t),u^*(t))$. As $(x^*(t),u^*(t)) \in D_{\epsilon(t)+\bar{u}_\fb(t)}$, $(x(t),u(t)) \in D$, for all $t \in [0, T]$. To apply this result for the learned NN controller derived using \eqref{eq:weak}, we note that Alg. \ref{alg:rrt} further ensures that $\bar\epsilon(t) \le \epsilon_\textrm{max}$, for all $t \in [0, T]$. As the preceding discussion shows that $\epsilon(t) \le \bar\epsilon(t)$ in execution, we have that $\epsilon(t) \le \epsilon_\textrm{max}$; therefore, $(\tilde x^*(t), x^*(t), u^*(t))$ remains within $\B_{\epsilon_\textrm{max}}(0) \times D$ in execution.
\end{proof}

\section{Results}

To demonstrate \lmtd on a wide range of systems, we show our method on a 4D nonholonomic car, a 6D underactuated quadrotor, and a 22D rope manipulation task. Throughout, we will compare with four baselines to show the need to both use the bound \eqref{eq:trk_bnd} and to remain within $D$, where the bound is accurate: B1) planning inside $D$ and assuming the model error is uniformly bounded by the average training error $\Vert d(t) \Vert \le \sum_{i=1}^N e_i /N$ to compute $\bar\epsilon(t)$, B2) planning inside $D$ and using the maximum training error $\Vert d(t) \Vert \le \max_{1\le i \le N} e_i$ as a uniform bound, B3) not remaining in $D$ in planning and assuming a uniform bound on model error $\Vert d(t) \Vert \le \max_{1\le i \le N} e_i$ in computing $\bar\epsilon(t)$ for collision checking, and B4) not remaining in $D$ and using our error bound \eqref{eq:trk_bnd}. We note that B3-type assumptions are common in prior CCM work \cite{sumeet_icra, dawei}. In baselines that leave $D$, the space is unconstrained: $\X = \mathbb{R}^{n_x}$, $\U = \mathbb{R}^{n_u}$. We set the FTG-based estimation probability $\rho=0.975$ for each constant. Please see Table \ref{table:stats_car} for planning statistics and \href{https://drive.google.com/file/d/1R1mE417aWerVO0zXISi8_wSM1IzYIP1M/view?usp=sharing}{\textcolor{blue}{https://tinyurl.com/lmtcdrrt}} for a supplementary video which overviews the method and visualizes our results.

\begin{figure}
    \centering
    \includegraphics[width=\linewidth]{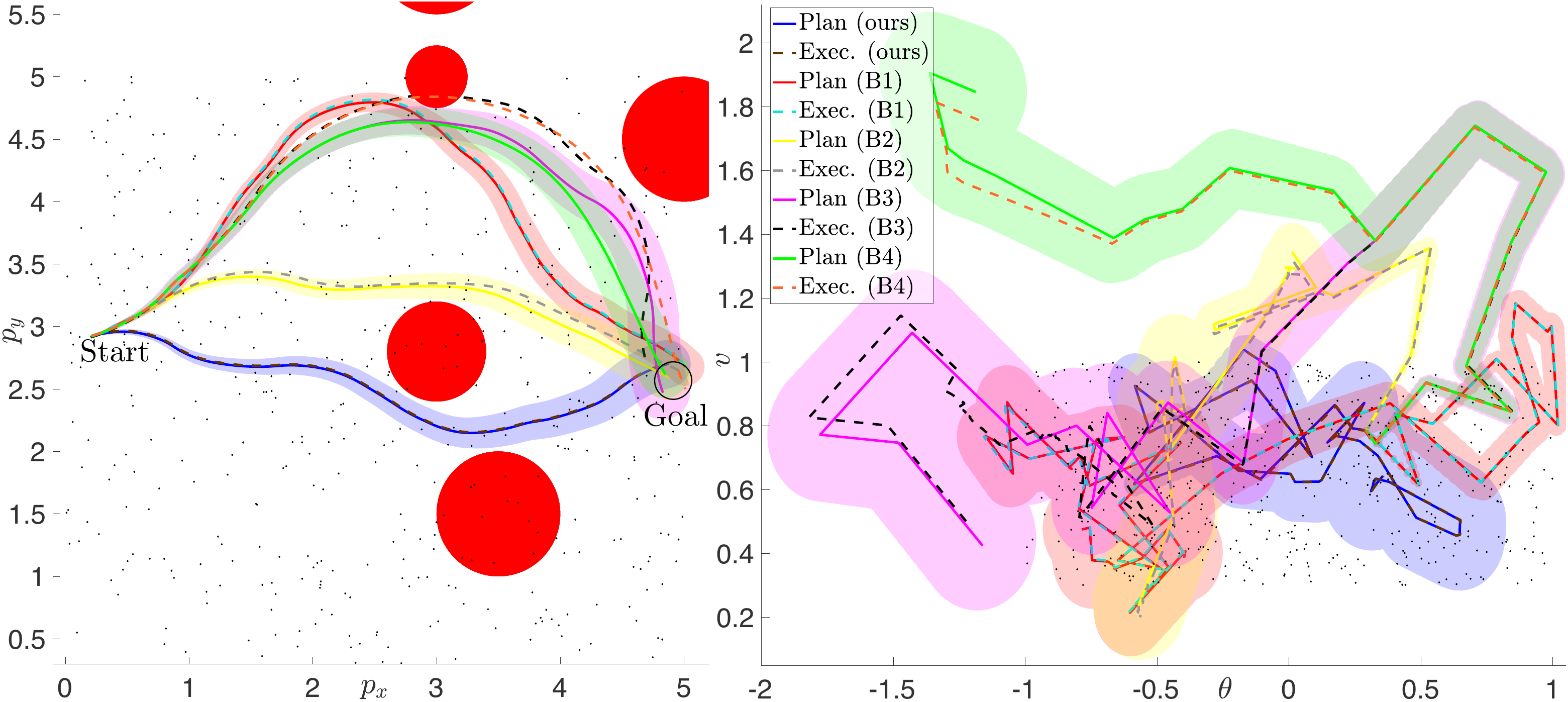}\vspace{-2pt}
    \caption{4D car; planned (solid lines) and executed trajectories (dotted lines). The filled red circles are obstacles. Tracking tubes for all methods are drawn in the same color as the planned trajectory. To aid in visualizing $D$, the small black dots are a subsampling of $\S$. We plot state space projections of the trajectories: \textbf{Left}: projection onto the $x, y$ coordinates; \textbf{Right}: projection onto the $\theta$, $v$ coordinates. For this example, LMTCD-RRT, B1, and B2 remain in $D$ in execution, while B3 and B4 exit $D$, and also exit their respective tracking tubes, leading to crashes.}
    \label{fig:car_results}
\end{figure}

\begin{table*}\centering
\begin{adjustbox}{max width=\textwidth}
\begin{tabular}{ c | c | c | c | c| c| c}
  & Avg. trk. error (Car) & Goal error  (Car)& Avg. trk. error (Quadrotor)& Goal error (Quadrotor)& Avg. trk. error (Rope)& Goal error (Rope)\\\hline
 \cellcolor{lightgray!50!}\lmtd \hspace{-7pt} & \cellcolor{lightgray!50!} 0.008 $\pm$ 0.004 (0.024) & \cellcolor{lightgray!50!} 0.009 $\pm$ 0.004 (0.023) & \cellcolor{lightgray!50!} 0.0046 $\pm$ 0.0038 (0.0186)& \cellcolor{lightgray!50!} 0.0062 $\pm$ 0.0115 (0.0873) & \cellcolor{lightgray!50!} 0.0131 $\pm$ 0.0063 (0.0278) & \cellcolor{lightgray!50!} 0.0125 $\pm$ 0.0095 (0.0352) \\ 
 \hspace{-7pt}B1: Mean, in $D$\hspace{-5pt} & 0.019 $\pm$ 0.012 (0.054) & 0.023 $\pm$ 0.016 (0.078) & 0.0052 $\pm$ 0.0051 (0.0311) & 0.0104 $\pm$ 0.0161 (0.0735) &  18.681 $\pm$ 55.917 (167.79) & 42.307 $\pm$ 126.81 (380.45) \\ 
 \cellcolor{lightgray!50!}\hspace{-7pt}B2: Max, in $D$\hspace{-5pt} & \cellcolor{lightgray!50!}0.02 $\pm$ 0.01 (0.05) [19/50] & \cellcolor{lightgray!50!}0.019 $\pm$ 0.012 (0.062) [19/50] & \cellcolor{lightgray!50!} --- [65/65] & \cellcolor{lightgray!50!} --- [65/65] & \cellcolor{lightgray!50!} 17.539 $\pm$ 52.380 (157.22) & \cellcolor{lightgray!50!}21.595 $\pm$ 64.295 (193.05) \\ 
 \hspace{-7pt}B3: Max, $\notin D$\hspace{-5pt} & 0.457 $\pm$ 0.699 (3.640) & 1.190 $\pm$ 1.479 (7.434) & 0.1368 $\pm$ 0.2792 (1.5408) & 0.8432 $\pm$ 1.3927 (9.0958) &  111.86 $\pm$ 39.830 (170.96) & 236.34 $\pm$ 72.622 (331.83) \\ 
 \cellcolor{lightgray!50!}\hspace{-7pt}B4: Lip., $\notin D$ \hspace{-5pt} & \cellcolor{lightgray!50!}0.704 $\pm$ 2.274 (13.313) & \cellcolor{lightgray!50!}2.246 $\pm$ 8.254 (58.32) & \cellcolor{lightgray!50!} 0.4136 $\pm$ 0.4321 (1.9466) & \cellcolor{lightgray!50!} 1.8429 $\pm$ 1.5260 (6.9859) & \cellcolor{lightgray!50!}17.301 $\pm$ 49.215 (148.43) & \cellcolor{lightgray!50!} 36.147 $\pm$ 52.092 (147.76)\\ 
\end{tabular}\vspace{-3pt}
\end{adjustbox}
\caption{\vspace{-3pt}Statistics for the car, quadrotor, and rope. Mean $\pm$ standard deviation (worst case) [if nonzero, number of failed trials].}\vspace{-12pt}
\label{table:stats_car}
\end{table*} 

\noindent \textbf{Nonholonomic car (4D)}: We consider the vehicle model $$\begin{bmatrix}\dot{p}_x \\ \dot{p}_y \\ \dot\theta \\ \dot v\end{bmatrix} = \begin{bmatrix}v\cos(\theta) \\ v\sin(\theta) \\ 0 \\ 0\end{bmatrix} + \begin{bmatrix} 0 & 0 \\ 0 & 0 \\ 1 & 0 \\ 0 & 1 \end{bmatrix}\begin{bmatrix}\omega \\ a\end{bmatrix},$$ where $u=[\omega,a]^\top$. As this model satisfies \eqref{eq:sparse}, we use the stronger CCM conditions \eqref{eq:strong}. We use 50000 training data-points uniformly sampled from $[0, 5]\times [-5,5]\times[-1,1]\times[0.3,1]$ to train $\gzero(x)$, $\gone(x)$, and $M(x)$, with the $x$ needed to train $M(x)$ coming directly from the state data in $\S$. We model $\gzero$ and $\gone$ as NNs, each with a single hidden layer of size 1024 and 16, respectively. We model $M(x)$ as an NN with two hidden layers, each of size 128. In training, we set $\underline w = 0.01$ and gradually increase $\alpha_1$ and $\alpha_2$ to $0.01$ and $10$, respectively. We select $r=0.6$ by incrementally growing $r$ as described in Sec. \ref{sec:method_verification}, collecting 5000 new datapoints for $\Psi$, giving us $\lambda = 0.09$, $L_{\f-\g} = 0.006$. $\delta_u = 1.01$, $\bar\lambda_D(M) = 0.258$, and $\underline\lambda_D(M) = 0.01$.

We plan for 50 different start/goal states in $D$, taking on average 6 mins, and compare against the four baselines. We visualize one trial in Fig. \ref{fig:car_results}. Over the trials, \lmtd and B2 never violate their respective bounds in any of the trials, while B1, B3, and B4 violate their bounds in 6, 48, and 43 of the 50 trials, respectively, which could lead to a crash (indeed, B3 and B4 crash in Fig. \ref{fig:car_results}). This occurs as the tracking error bounds for the baselines are invalid, as the baselines' model error bounds underestimate the true model error which could be seen in execution. Moreover, planning is infeasible in 19/50 of B2's trials, because the large uniform error bound can make it impossible to reach a goal while remaining in $D$. This suggests the utility of a fine-grained disturbance bound like \eqref{eq:disturbance_bnd}, especially when planning in $D$, which is quite constrained. Note that while B2 does not crash in this particular example, using the maximum training error can still be unsafe (as will be seen in later examples), as the true error can be higher in $D \setminus \S$. Finally, we note that the tracking accuracy difference between \lmtd and B1/B2 reflects that the spatially-varying error bound steers \lmtd towards lower error regions in $D$. Overall, this example suggests that \eqref{eq:trk_bnd} is accurate, while coarser disturbance bounds or exiting $D$ can be unsafe.

\begin{figure}
    \centering
    \includegraphics[width=\linewidth]{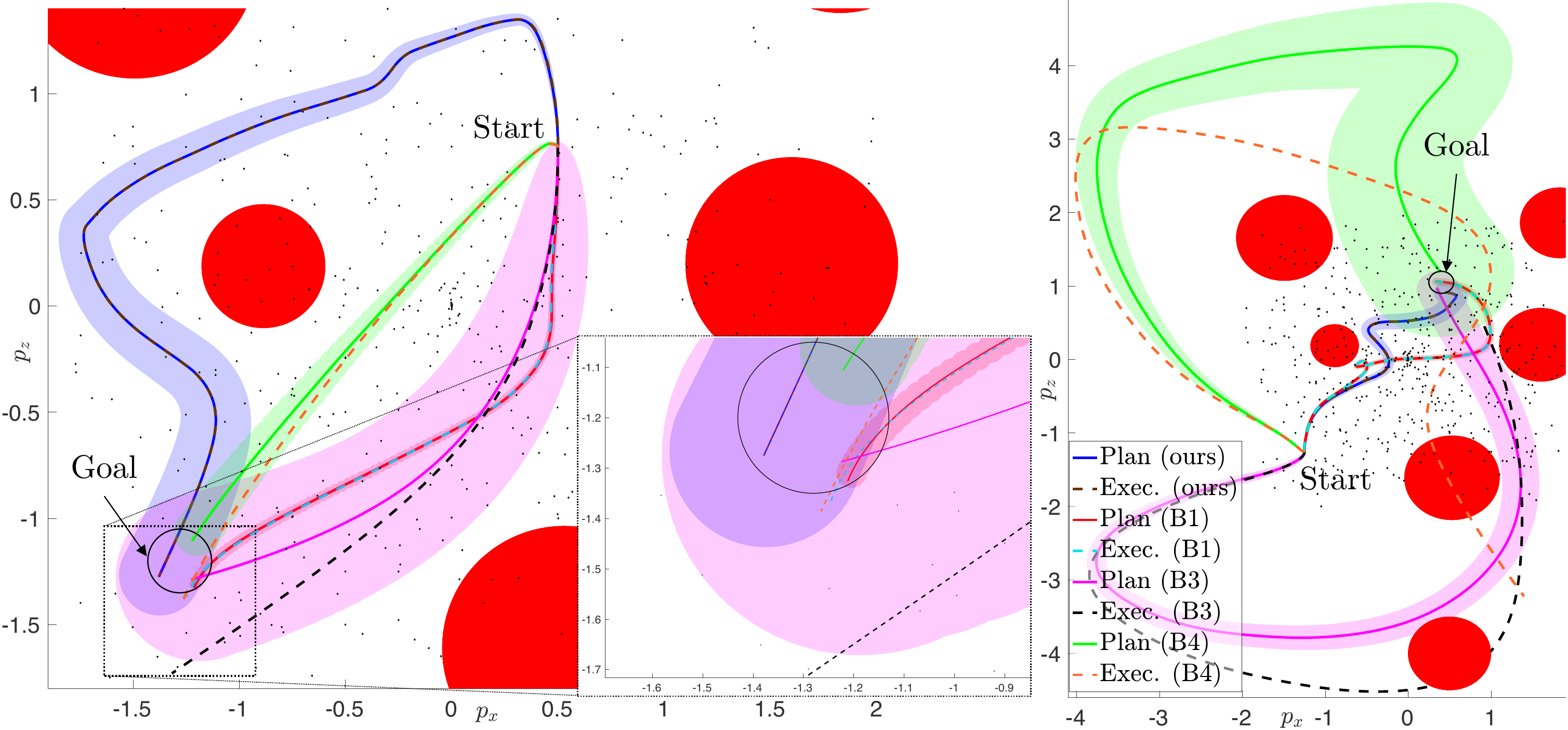}\vspace{-2pt}
    \caption{6D quadrotor; planned (solid lines) and executed trajectories (dotted lines). The filled red circles are obstacles. Tracking tubes for all methods are drawn in the same color as the planned trajectory. The small black dots are a subsampling of $\S$. We plot state space projections of the trajectories onto the $p_x$, $p_z$ coordinates. \textbf{Left}: for this example, LMTCD-RRT remains within its tracking tube, and all baselines violate their respective bounds near the end of execution (see inset). \textbf{Right}: for this example, LMTCD-RRT remains within its tracking tube, and B3 and B4 exit $D$ and crash.}
    \label{fig:quad_results}
\end{figure}

\noindent\textbf{Underactuated planar quadrotor (6D)}: We consider the quadrotor model in \cite[p.20]{sumeet_icra} with six states and two inputs:$$\begin{bmatrix}\dot{p}_x \\ \dot{p}_y \\ \dot{\phi} \\ \dot{v}_x \\ \dot{v}_z \\ \ddot{\phi} \end{bmatrix} = \begin{bmatrix} v_x \cos(\phi) - v_z \sin(\phi) \\ v_x \sin(\phi) + v_z \cos(\phi) \\ \dot\phi \\ v_z \dot\phi - g\sin(\phi) \\ -v_x \dot\phi - g\cos(\phi) \\ 0\end{bmatrix} + \begin{bmatrix}0 & 0 \\ 0 & 0 \\ 0 & 0 \\ 0 & 0 \\ 1/m & 1/m \\ l/J & -l/J\end{bmatrix}\begin{bmatrix}u_1 \\ u_2\end{bmatrix},$$ where $x = [p_x, p_z, \phi, v_x, v_z, \dot\phi]$, modeling the linear/angular position and velocity, and $u=[u_1,u_2]$, modeling thrust. We use the parameters $m = 0.486$, $l = 0.25$, and $J = 0.125$. These dynamics also satisfy \eqref{eq:sparse}, so we use the stronger CCM conditions \eqref{eq:strong}. We sample 245000 training points from $[-2, 2]\times [-2,2]\times[-\pi/3,\pi/3]\times[-1,1]\times[-1,1]\times[-\pi/4,\pi/4]$ to train $\gzero(x)$, $\gone(x)$, and $M(x)$, with the $x$ data for training $M(x)$ being the state data in $\S$. We model $\gzero$ and $\gone$ as NNs with a single hidden layer of size 1024 and 16, respectively. We model $M(x)$ as an NN with two hidden layers of size 128. In training, we set $\underline w = 0.01$ and gradually increase $\alpha_1$ and $\alpha_2$ to $0.001$ and $0.33$, respectively. We select $r=1.0$ by incrementally growing $r$ as in Sec. \ref{sec:method_verification}, resulting in 10000 new datapoints for $\Psi$. This gives us $\lambda = 0.09$, $L_{\f-\g} = 0.007$. $\delta_u = 1.9631$, $\bar\lambda_D(M) = 4.786$, and $\underline\lambda_D(M) = 0.0909$.

We plan for 65 different start/goal states within $D$, taking 1 min on average, and compare against the baselines. We visualize two trials in Fig. \ref{fig:quad_results}. Our attempts to run B2 failed, as the error bound was too large to feasibly plan within $D$ in all 65 trials, again suggesting the need for a local model error bound. Over these trials, \lmtd never violates its computed bound in execution, while B1, B3, and B4 violate their bounds 14/65, 32/65, and 65/65 times, respectively. As for the car example, these bounds are violated because the model error descriptions assumed by these baselines can underestimate the true model error seen in execution. From Table \ref{table:stats_car}, one can see that \lmtd obtains the lowest error, though it is closely matched by B1. However, LMTCD-RRT never violates the tracking tubes in execution, while B1 does (i.e. Fig. \ref{fig:quad_results}, left). B3-B4 perform poorly, with the controller failing to overcome the model error, causing crashes (Fig. \ref{fig:quad_results}, right). Overall, this example highlights the need for LMTCD-RRT's local error bounds while demonstrating our method's applicability to highly-underactuated systems.

\noindent\textbf{10-link rope (22D)}: To demonstrate that our method scales to high-dimensional, non-polynomial systems well beyond the reach of SoS-based methods, we consider a planar rope manipulation task simulated in Mujoco \cite{todorov2012mujoco}. We consider a 10-link (11-node) rope approximation, where each link can stretch, and the head of the rope (see Fig. \ref{fig:rope_results}(d)) is velocity-controlled. The system has 22 states: the first two contain the $xy$ position of the head, and the rest are the $xy$ positions of the other nodes, relative to the head, and the system has two controls for the commanded $xy$ head velocities. We wish to steer the \textit{tail} of the rope to a given $xy$ region while ensuring the rope does not collide in execution (c.f. Fig. \ref{fig:rope_results}). This is a challenging task, as the tail of the rope is highly underactuated, and steering it to a goal requires modeling the complicated friction forces that the rope is subjected to. We collect three demonstrations to train the dynamics (see the supplementary video for a visualization): in the first, the rope begins horizontally and performs a counterclockwise elliptical loop; the second starts vertically and moves up, the third begins horizontally and moves right, giving a total of 20500 datapoints. As the rope dynamics do not satisfy \eqref{eq:sparse}, we learn both $M(x)$ and $u(\tilde x,x^*,u^*)$; to do so, we sample 20500 state/control perturbations around the demonstrations and evaluate the dynamics at these points, giving $|\S| = 41000$. We model $\gzero$ and $\gone$ as three-layer NNs of size 512. $M(x)$ is modeled with two hidden layers of size 128, and $u(\tilde x,x^*,u^*)$ is modeled with a single hidden layer of size 128. In training, we set $\underline w = 1.0$ and gradually increase $\alpha_1$ and $\alpha_3$ to $0.005$ and $0.56$, respectively. To ensure that the CCM and controller are invariant to translations of the rope, we enforce $M(x)$ and $u(\tilde x,x^*,u^*)$ to not be a function of the head position. To simplify the dynamics learning, we note that as the head is velocity-controlled, it can be modeled as a single-integrator; we hardcode this structure and learn the dynamics for the other 20 states. We obtain $\epsilon_\textrm{max} = 0.105$ and select $r = 0.5$ by incrementally growing $r$ as in Sec. \ref{sec:method_verification}, resulting in $|\Psi| = 10000$, $\lambda = 0.0625$, $L_{\f-\g} = 0.023$, $\bar u_\fb = 0.249$, $\bar\lambda_D(M) = 3.36$, and $\underline\lambda_D(M) = 1$.

We plan for 10 different start/goal states within $D$, taking 9 min on average, and compare against the baselines. As this example uses the learned controller \eqref{eq:tanh_controller}, we adapt the baselines so that B1 and B2 remain in $\B_{\epsilon_\textrm{max}}(0) \times D$, while B3 and B4 are unconstrained. We visualize one task in Fig. \ref{fig:rope_results}: the rope starts horizontally, with the head at $[0,0]$, and needs to steer the tail to $[3,0]$, within a $0.15$ tolerance. \lmtd stays very close to the training data, reaching the goal with small tracking tubes. B1 and B2 also remain close to the training data as they plan in $D$, but as both the mean and maximum bounds may underestimate the true model error in $D$, they stray too close to the boundary of $D$, and the larger model error pushes them out of $D$, causing the system to become unstable as the learned $u$ applies large inputs in an attempt to stabilize around the plan. B3 exploits errors in the model, planning a trajectory which is highly unrealistic. This is allowed to happen because the maximum error severely underestimates the model error outside of $D$, leading to a major underestimate of the tracking error that would be seen in execution. When executing, the system immediately goes unstable due to the large distance between the plan and the training data. The plan from B4 is forced to remain close to the training data at first, in order to move through the narrow passage. This is because the Lipschitz bound, while an underestimate outside of $D$, still grows quickly with distance from $\S$; attempting to plan a trajectory similar to B3 fails, since the tracking error tube grows so large in this case that it becomes impossible to reach the goal without the tube colliding with the obstacles. After getting through the narrow passage, B4 drifts from $D$ and correspondingly fails to be tracked beyond this point. Over these 10 trials, \lmtd never violates the computed tracking bound, and B1, B2, B3, and B4 violate their bounds in 10, 6, 10, and 9 trials out of 10, respectively. Overall, this result suggests that contraction-based control can scale to very high-dimensional systems (i.e. deformable objects) if one finds where the model and controller are good and takes care to stay there during planning and execution.

	\begin{figure}	
    \centering	\vspace{-3pt}
    \includegraphics[width=\linewidth]{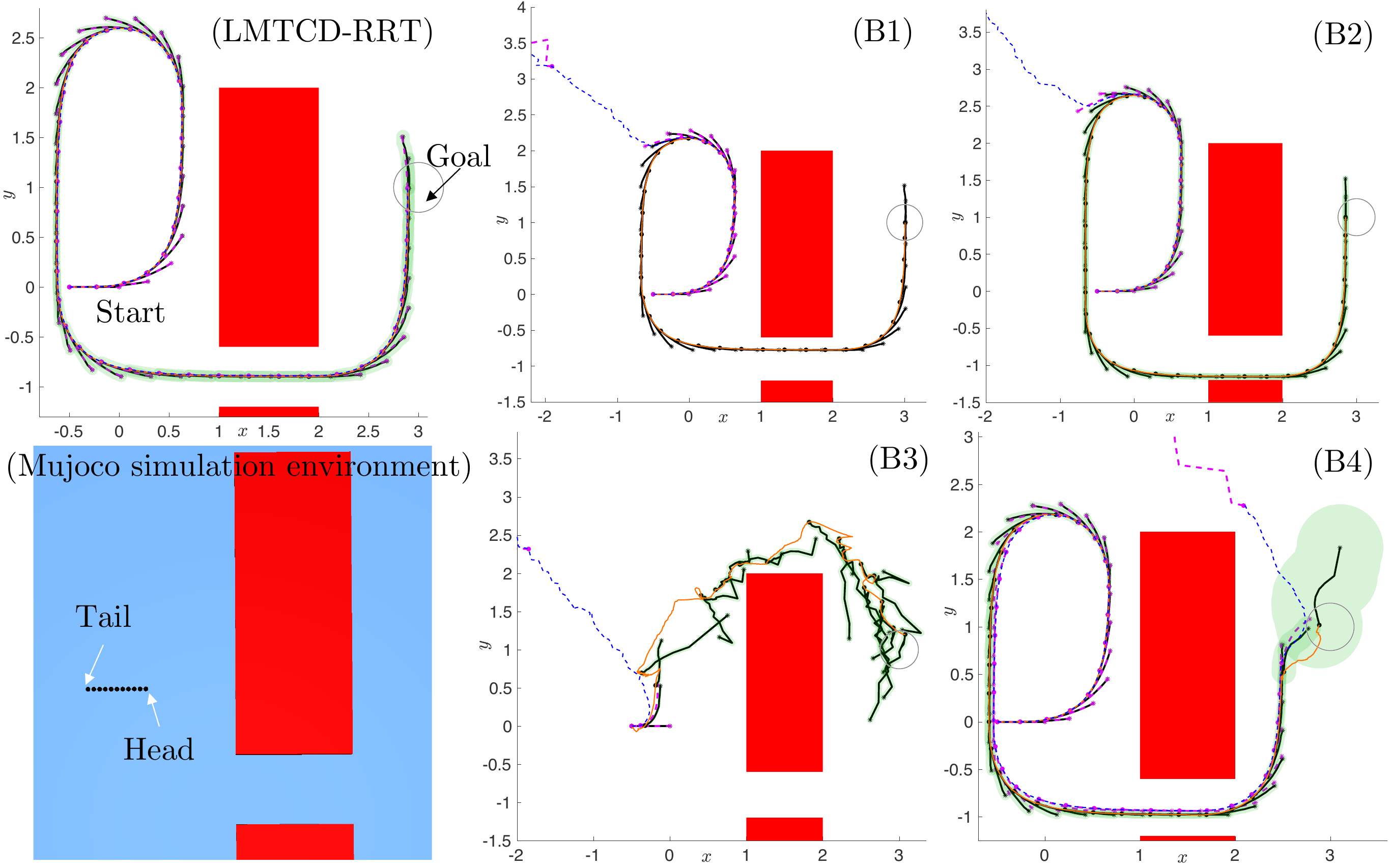}\vspace{-8pt}
    \caption{22D planar rope dragging task. Snapshots of the planned trajectory are in black, snapshots of the executed trajectory are in magenta, and the tracking error tubes are in green. For further concreteness, for each snapshot, we mark the head of the rope with an asterisk, and we mark the tail of the rope with a solid dot. Additionally, the trajectory of the tail in the plan is plotted in orange, while the trajectory of the tail in execution is plotted in blue. Only LMTCD-RRT reaches the goal, while all baselines become unstable when attempting to track their respective plans. We also show the original Mujoco simulation environment in the bottom left.}
    \label{fig:rope_results}	
\end{figure}

\section{Conclusion} 
\label{sec:conclusion}

We present a method for safe feedback motion planning with unknown dynamics. To achieve this, we jointly learn a dynamics model, a contraction metric, and contracting controller, and analyze the learned model error and trajectory tracking bounds under that model error description, all within a trusted domain. We then use these tracking bounds together with the trusted domain to guide the planning of probabilistically-safe trajectories; our results demonstrate that ignoring either component can lead to plan infeasibility or unsafe behavior. Future work involves extending our method to plan safely with latent dynamics models learned from image observations.

\section*{Acknowledgments}
We deeply thank Craig Knuth for insightful discussions and for feedback on the manuscript. This work was supported in part by NSF grants IIS-1750489 and ECCS-1553873, ONR grants N00014-21-1-2118 and N00014-18-1-2501, and a National Defense Science and Engineering Graduate (NDSEG) fellowship.

\bibliographystyle{IEEEtran}
\bibliography{references.bib}

\end{document}